\documentclass[letterpaper]{article} 
\usepackage{aaai2027}                
\usepackage[hyphens]{url}
\usepackage{graphicx}
\usepackage{natbib}
\usepackage{caption}
\usepackage{booktabs}
\usepackage{amsmath}
\usepackage{comment}
\usepackage{amssymb}
\newtheorem{proposition}{Proposition}
\newtheorem{corollary}{Corollary}

\title{HKVLM: Faithful Query--Region Binding for Frozen-Detector Visual Grounding}

\author{Bo Ma}
\affiliations{
    Auckland University of Technology
}

\begin{document}
\maketitle

\begin{abstract}
Visual grounding often fails even when the target object is present in the proposal pool, because the language-side referent is bound to the wrong region. We study this binding failure under frozen perception and ask whether an explicit query--region alignment hook, together with a perception-grounded abstention mechanism, can improve faithful grounding without retraining the detector or the vision-language backbone. HKVLM freezes a language-aligned open-vocabulary detector for localization and learns a lightweight hook that maps referential query embeddings to detector proposals in a shared space; a verifier abstains when no region sufficiently supports the query. We prove an exact proposal-level diagnostic decomposition, $(1-\mathrm{SeeErr})(1-\mathrm{SayErr})$, separating proposal-coverage failures from conditional binding failures, and a monotonicity result that characterizes the faithfulness--recall trade-off induced by abstention.

Across RefCOCO, RefCOCO+, RefCOCOg, and POPE, HKVLM improves over untrained and trained matched-perception binding controls and substantially reduces hallucination through abstention. Strong coordinate-decoding and end-to-end fine-tuned baselines remain much higher in raw grounding accuracy, and a reasoning-stress set exposes binding as the main current bottleneck. We therefore present HKVLM as a diagnostic and mechanism-level study of query--region binding under frozen perception, not as an absolute localization leader.
\end{abstract}

\section{Introduction}
Practitioners who adapt small open-source VLMs for object detection report a
stubborn pattern: out of the box the models localize poorly, and modest
fine-tuning---for example low-rank adaptation on a thousand images---yields
only marginal gains. The models clearly attend to the right pixels yet still
return wrong labels or imprecise boxes: they see correctly but mis-speak.

The gap is especially visible in binding-heavy visual grounding: requests phrased as intent
or world knowledge rather than a category name---``something to cut paper with'',
``the unoccupied chair'', ``the person not wearing a helmet''---where a pure
detector has no phrase to match and the system must connect language intent to
the correct region. This is the regime we study, rather than closed-set
detection, where specialized detectors already excel.

A common explanation blames the visual encoder and proposes a stronger one.
Encoder upgrades help only marginally in practice, which points elsewhere. In
a VLM the encoder output is consumed by the language model exactly as text is:
keys and values are read according to a query, and the quality of that query
is not guaranteed. When the box itself is produced by the autoregressive
decoder---as a sequence of coordinate tokens \citep{kosmos2,shikra,griffon} or
quantized coordinate symbols \citep{qwen25vl,rexomni}---the geometry of the
object is forced through the same brittle channel that produces words. A single
bad query token can override a correct percept. The localization signal and the
language signal are entangled.

We take the position that faithful grounding requires \emph{three} separable
competencies: \textbf{localization competence} (place a tight box on an
object), \textbf{referential faithfulness} (name only what is actually there),
and \textbf{binding} (connect the user's intent to the correct region).
Coordinate-as-text designs entangle all three inside one autoregressive head.
Two-stage reasoning pipelines \citep{detgpt,lisa} separate localization but
re-couple through a lossy text interface: the model's intent is compressed into
a class string handed to an external detector, discarding instance-specific and
relational cues. Recent decoupled designs recast box regression as retrieval
over proposals \citep{chatrex,groma}, a real step forward, but selection still
happens inside the autoregressive language model via index tokens and there is
no explicit objective that ties language semantics to region geometry or that
penalizes ungrounded assertions.

We propose HKVLM, which keeps localization entirely outside the language path.
A frozen, language-aligned open-vocabulary detector
\citep{groundingdino,glip} produces class-agnostic, high-recall region proposals
together with vision-language-fused region embeddings. The language model reads
the image and the reasoning instruction and emits one or more \emph{referential
query embeddings} that carry its world knowledge. A small trained module---the \emph{alignment hook}---projects
queries and regions into a shared space and binds them by contrastive retrieval
and bipartite assignment, reusing the set-prediction machinery of DETR
\citep{detr}. Categories and answers are produced in language, but every box
comes from the detector and every emitted category must be justified by an
aligned region whose embedding agrees. This last requirement is a
\emph{perception-grounded faithfulness} constraint, and it is motivated by
evidence that grounding alone does not automatically reduce hallucination
\citep{objgndhalluc}.

Freezing the heavy components has a practical payoff. Only the hook and a thin
projection are trained, so the cold-start regime that defeats parameter-efficient
tuning of a monolithic VLM becomes tractable: the localization stack already
works, and learning is confined to the binding.

\paragraph{What this paper claims and what it does not claim.}
We do not claim state-of-the-art absolute localization. We ask whether, given a fixed frozen detector, explicit query--region binding and a perception-grounded abstention mechanism reduce conditional binding errors and unsupported assertions. HKVLM is therefore a matched-perception mechanism study.

Our contributions are methodological, theoretical, and diagnostic:
\begin{enumerate}
\item A \emph{say-vs-see} diagnostic decomposition (Section~\ref{sec:formulation})
that separates perceptual error (SeeErr: detector fails to propose the target) from
binding error (SayErr: hook selects the wrong proposal given the target is available),
enabling principled attribution of grounding failures. We prove that these two
quantities compose \emph{exactly}: at the proposal level, grounding accuracy equals
$(1-\mathrm{SeeErr})(1-\mathrm{SayErr})$ (Proposition~\ref{prop:decomp}), and a
monotonicity corollary (Corollary~\ref{cor:veto}) shows any perception-grounded
veto weakly decreases present-query grounding, formally characterizing the
faithfulness--recall trade-off we measure empirically.
\item HKVLM (Section~\ref{sec:method}): a grounding framework in which a frozen detector
owns localization and referential query embeddings are bound to region proposals by
explicit shared-space alignment and a perception-grounded veto, rather than by
autoregressive coordinate or index tokens.
\item Empirical evidence (Section~\ref{sec:results}) that, under \emph{matched frozen
perception}, the learned alignment hook substantially outperforms untrained cross-space
matching and the faithfulness veto substantially reduces hallucination on
random and popular existence probes. We present these
results as mechanism-level evidence at a fixed perceptual budget, not as a
state-of-the-art absolute-localization claim.
\end{enumerate}

\section{Related Work}
\paragraph{Coordinate-as-text grounding VLMs.}
Kosmos-2 \citep{kosmos2} encodes boxes as location tokens and links text spans
to regions; Shikra \citep{shikra} renders coordinates as natural-language
numbers; Ferret \citep{ferret} mixes discrete points with continuous region
features; Griffon \citep{griffon,griffong} spells out dense locations as text;
and the Qwen-VL family \citep{qwenvl,qwen2vl,qwen25vl,qwen3vl} decodes
absolute coordinates. Rex-Omni \citep{rexomni} pushes this paradigm furthest,
casting detection as next-token prediction over quantized coordinate symbols and
reaching detector-level zero-shot accuracy. All of these route geometry through
the autoregressive head. \emph{We differ} by removing localization from that
head entirely; the language model never emits a coordinate.

\paragraph{Reasoning pipelines with a class-string interface.}
DetGPT \citep{detgpt} has an MLLM name target objects that an open-vocabulary
detector then localizes; LISA \citep{lisa} decodes a \texttt{<SEG>} embedding
into a mask through SAM \citep{sam}. These separate localization but communicate
intent through a single string or token, which is lossy for instance-specific
and relational queries. \emph{We differ} by binding a continuous referential
query embedding to region embeddings, preserving more of the intent.

\paragraph{Decoupled retrieval and localized tokenization.}
ChatRex \citep{chatrex} turns box regression into retrieval with a Universal
Proposal Network and has the LLM select proposals by index; Groma
\citep{groma} localizes through a region tokenizer referenced by index. Both
validate the regression-to-retrieval move that we build on. Referring variants
extend this line \citep{rexseek}, and grounded conversational models attach
grounding to chat \citep{llavagrounding,glamm}. \emph{We differ} in
where and how selection happens: our binding is an explicit shared-space
alignment trained with a matching and faithfulness objective, not an index token
emitted by the autoregressive model, and our recipe targets cold-start
adaptation rather than large-scale co-training.

\paragraph{Grounding frozen multimodal models.}
F-LMM \citep{flmm} keeps the LMM frozen and turns its word--image attention maps
into masks with a few trainable layers and a SAM refiner, preserving
conversational ability. We share the frozen-stack, light-adapter philosophy.
\emph{We differ} by sourcing localization from an external detector bound by
explicit alignment, and by adding a faithfulness constraint; F-LMM derives
localization from internal attention and targets segmentation. A broader trend
injects MLLM hidden states as priors to condition detector queries
\citep{mllmovdreview}; our binding is an explicit retrieval-plus-assignment
objective with a perceptual veto rather than a soft prior.

\paragraph{Classical and end-to-end visual grounding.}
Our work is also related to referring expression comprehension and visual grounding, including end-to-end systems such as MDETR \citep{mdetr} and TransVG \citep{transvg}. These models optimize absolute localization through joint multimodal training or coordinate regression. HKVLM instead freezes perception and asks whether explicit query--region binding and abstention improve faithful grounding under a fixed proposal pool; we use end-to-end REC only as absolute-performance context.


\paragraph{Benchmark quality.}
Recent work \citep{refl4} documents annotation noise in the RefCOCO family and
proposes cleaned evaluation splits. We report standard splits for direct
comparability and include Ref-L4 as a noise-audited stress check: HKVLM reaches
0.235 on Ref-L4 val without Ref-L4-specific tuning, while the same frozen
detector queried directly with the expression reaches 0.662. We report this as
interface-dependent headroom; full interpretation is in the appendix.

\section{Problem Formulation}
\label{sec:formulation}
Given an image $I$ and a free-form request $q$ (a category name, a referring
expression, or an instruction), the system must return a set of
category--box pairs $\{(c_i, b_i)\}$ that are both \emph{accurate} (tight,
correctly labeled) and \emph{faithful} (no $c_i$ without supporting evidence in
$I$).

We separate three competencies. \textbf{Localization competence} is the ability
of a perception module to produce a region set $\mathcal{R}=\{r_1,\dots,r_M\}$,
with boxes $\{b_j\}$ and embeddings $\{e_j\}$, that covers the true objects; we
measure it by proposal recall (the fraction of true objects covered by the top $M$ proposals). \textbf{Binding} is the map
from the request $q$ to the correct subset of $\mathcal{R}$. \textbf{Referential
faithfulness} is the property that every emitted category is supported by a bound
region.

\paragraph{Say-vs-see decomposition.}
The same grounding failure can arise because the right region was never proposed
(\emph{see} error) or because the right region was available but mis-bound
(\emph{say} error). Let $b_n^\star$ denote the ground-truth box for present
query $n$, and let $\hat{b}_n$ be the selected proposal box. We define
\begin{align}
z_n(\tau) &= \mathbb{1}\!\left[\max_j \mathrm{IoU}(b_{n,j}, b_n^\star)\ge \tau\right], \quad y_n^\star=1,\\
\mathrm{SeeErr}(\tau) &= 1 - \frac{1}{N_+}\sum_{n:y_n^\star=1} z_n(\tau),\\
\mathrm{SayErr}(\tau) &= \frac{\sum_{n:y_n^\star=1} z_n(\tau)\,\mathbb{1}\!\left[\mathrm{IoU}(\hat{b}_n,b_n^\star)<\tau\right]}{\sum_{n:y_n^\star=1} z_n(\tau)},
\end{align}
where $N_+$ is the number of present queries. The proposed binding error is the
conditional failure rate given proposal availability; we treat $\mathrm{BindErr}$
as an exact alias for $\mathrm{SayErr}$. We define these errors over the detector
proposal set so that perception and binding are measured on the same box-level
event. Coordinate-as-text VLMs conflate the two; HKVLM is designed to drive
$\mathrm{SayErr}$ down without inflating $\mathrm{SeeErr}$, because perception
is frozen and strong.

The following identity makes the decomposition exact rather than merely
descriptive, so that the two error sources can be attributed and optimized
independently.

\begin{proposition}[Exact grounding decomposition]
\label{prop:decomp}
Fix an IoU threshold $\tau$ and measure grounding at the proposal level, i.e.\
the emitted box for present query $n$ is the selected proposal
$\hat{b}_n=b_{n,\hat{s}(n)}$ with $\hat{s}(n)=\arg\max_j s_{nj}$. Let grounding
accuracy be $\mathrm{Acc}(\tau)=\frac{1}{N_+}\sum_{n:y_n^\star=1}
\mathbf{1}[\mathrm{IoU}(\hat{b}_n,b_n^\star)\ge\tau]$. Then
\begin{equation}
\mathrm{Acc}(\tau)=\big(1-\mathrm{SeeErr}(\tau)\big)\big(1-\mathrm{SayErr}(\tau)\big).
\label{eq:decomp}
\end{equation}
\end{proposition}

Unless otherwise stated, the diagnostic decomposition is evaluated at the
proposal-selection level, before any optional box-residual refinement. When
residual refinement is enabled, the expression should be interpreted as a
proposal-level diagnostic rather than an identity over final refined boxes.

\noindent\emph{Proof sketch.} Because $\hat{b}_n$ is one of the $M$ proposals,
$\mathrm{IoU}(\hat{b}_n,b_n^\star)\le\max_j\mathrm{IoU}(b_{n,j},b_n^\star)$, so a
correct hit ($h_n{=}1$) implies coverage ($z_n{=}1$); hence $h_n=z_n h_n$.
Summing, $\sum_n h_n=\sum_n z_n-\sum_n z_n(1-h_n)=(\sum_n z_n)(1-\mathrm{SayErr})$.
Dividing by $N_+$ gives \eqref{eq:decomp}. A full proof is in the appendix.\hfill$\square$

\begin{corollary}
\label{cor:veto}
{\upshape\bfseries(Veto monotonicity and the faithfulness--recall trade-off.)}
Let the perception-grounded veto emit present query $n$ only when a gate
$g_n\in\{0,1\}$ fires. The emitted grounding accuracy is
$\mathrm{Acc}_{\mathrm{veto}}(\tau)=\mathrm{Acc}(\tau)-\frac{1}{N_+}
\sum_{n:y_n^\star=1}(1-g_n)h_n\le\mathrm{Acc}(\tau)$, with equality iff the veto
never suppresses a correctly grounded present query.
\end{corollary}

\noindent This makes the drop from Table~\ref{tab:grounding}'s no-veto column to
the HKVLM column an exact accounting identity, not an artifact: it is precisely
the mass of correctly grounded present queries the veto abstains on in exchange
for the faithfulness gains of Table~\ref{tab:pope}.

\section{Method: HKVLM}
\label{sec:method}
HKVLM has four parts: a frozen perception module, a language module that emits
referential queries, the trained alignment hook, and a faithfulness objective.
Figure~\ref{fig:arch} shows the flow, and
the full paradigm taxonomy is provided in the appendix.

\begin{figure*}[t]
\centering
\includegraphics[width=0.7\textwidth]{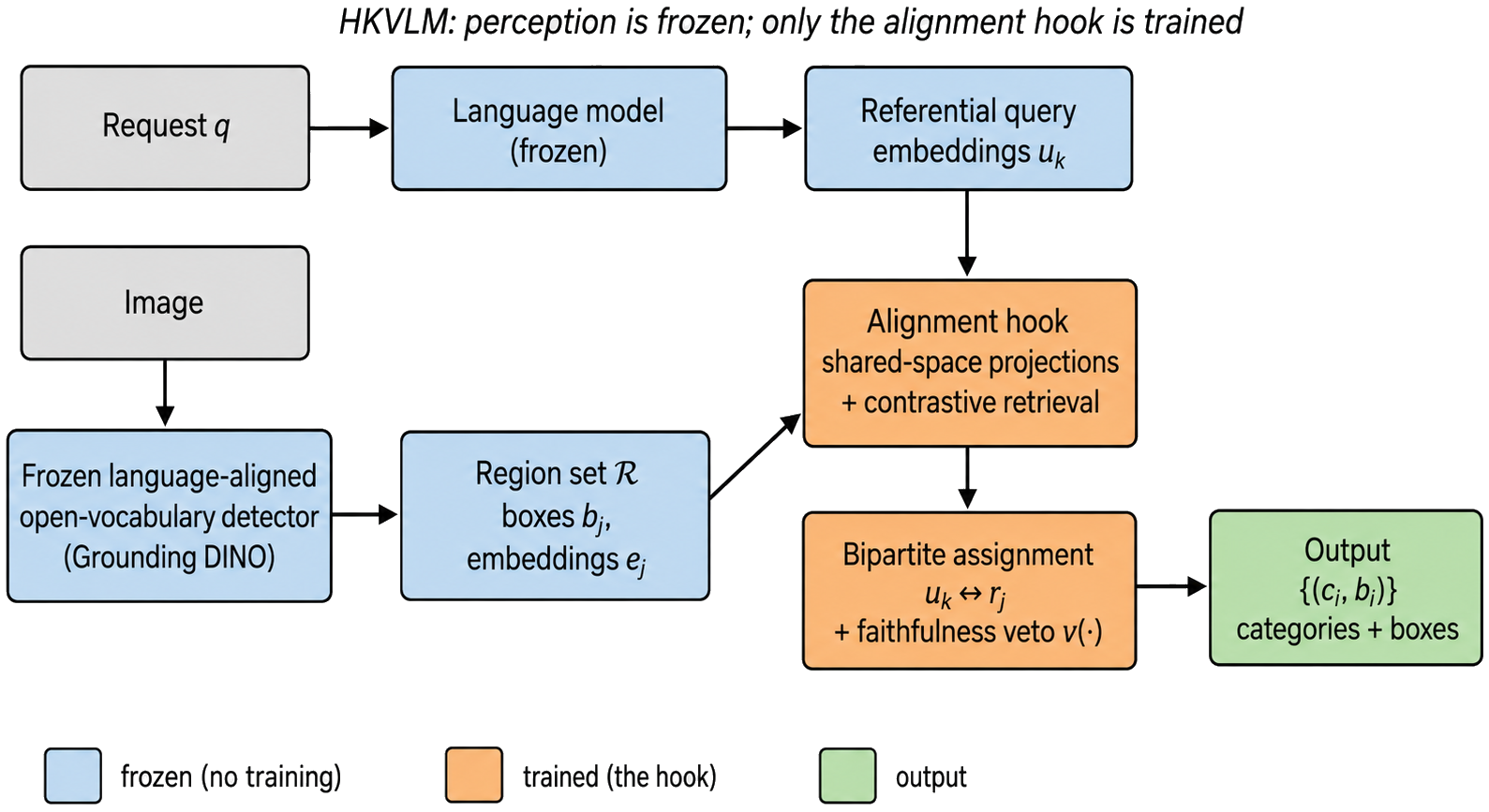}
\caption{\textbf{HKVLM overview.} Localization is owned by a frozen,
language-aligned open-vocabulary detector (Grounding DINO in our implementation)
and never passes through the language model. The language model (Qwen2.5-VL-7B,
frozen) turns the request into referential query embeddings; the trained
alignment hook projects queries and region embeddings into a shared space,
performs retrieval and assignment, and applies a perception-grounded faithfulness
veto. Only the hook parameters (orange) are trained.}
\label{fig:arch}
\end{figure*}

\subsection{Frozen perception module}
The perception module is a frozen, pre-trained \emph{language-aligned}
open-vocabulary detector. We use Grounding DINO \citep{groundingdino} (GLIP
\citep{glip} is an alternative) and take its class-agnostic, high-recall set of
$M$ proposals $\mathcal{R}=\{(b_j,e_j)\}_{j=1}^{M}$, where $b_j$ is a box and
$e_j\in\mathbb{R}^{d}$ is the detector's vision-language-fused region feature for
that proposal. The choice is deliberate: because these region features already
live in a space that has been aligned to text during grounded pre-training,
binding them to the language model's query embeddings is a short bridge rather
than a cross-modal leap. A purely self-supervised backbone (DINOv2/DINOv3
\citep{dinov2,dinov3}), whose features carry no language semantics, would be
expected to make binding harder, as there is no shared language anchor to align with. The module is frozen,
so it carries all localization competence and box quality never depends on
language decoding; operating class-agnostically maximizes recall and decouples
``what is here'' from ``which one is meant.''

\subsection{Referential query embeddings}
The language model receives $I$ and $q$ and produces $K$ referential query
embeddings $\{u_k\}_{k=1}^{K}$, one per entity the request implies. We read these
from dedicated query slots appended to the prompt, so the model's reasoning and
world knowledge shape $u_k$ while the language model itself can remain frozen or
be adapted with low-rank updates. Crucially, $u_k$ is a continuous vector rather
than a coordinate or index token; it can preserve instance-specific intent that a
class string may discard.

\subsection{Alignment hook: binding by retrieval and assignment}
Two light projections map regions and queries into a shared space,
$\tilde{e}_j = f_\theta(e_j)$ and $\tilde{u}_k = g_\phi(u_k)$. The affinity
between query $k$ and region $j$ is $s_{kj} = \langle \tilde{u}_k,
\tilde{e}_j\rangle / \kappa$ with temperature $\kappa$. Binding combines two
objectives.

\emph{Contrastive retrieval} pulls each query toward its matched region(s) and
pushes it from the rest:
\begin{equation}
\mathcal{L}_{\mathrm{ret}} = -\sum_{k}\log
\frac{\sum_{j\in P(k)} \exp(s_{kj})}{\sum_{j=1}^{M}\exp(s_{kj})},
\end{equation}
where $P(k)$ is the set of positive regions for query $k$.

\emph{Assignment} resolves which proposal each query claims. In this paper's
referring-grounding experiments each expression yields a single query
($K{=}1$), so assignment is top-1 selection over $M$ proposals and the dominant
runtime is the affinity computation; the matched pair supervises
$\mathcal{L}_{\mathrm{ret}}$ and a light box-refinement residual on
$b_{\sigma(k)}$. Only $\theta,\phi$ (and optionally LoRA in the language model)
are learned. A general multi-entity ($K{>}1$) formulation via Hungarian
matching \citep{detr} is specified in the appendix, together with a small
gRefCOCO demonstration showing partial multi-target recall but low exact set
recovery without multi-target training.

\subsection{Perception-grounded faithfulness}
A query may correspond to no object. We attach a verification head $v(\cdot)$
that, for the best-matched region of each query, predicts whether the bound pair
is genuinely supported, supervised by whether a true object exists:
\begin{equation}
\mathcal{L}_{\mathrm{faith}} = \sum_{k}\mathrm{BCE}\big(v(\tilde{u}_k,\tilde{e}_{\sigma(k)}),\, y_k\big),
\end{equation}
with $y_k=1$ when query $k$ has a real referent. At inference, a category is
emitted only when its bound region passes the verifier and the affinity exceeds
a threshold; otherwise the model abstains. This is the perceptual veto: the
language model cannot assert an object that perception does not support, directly
targeting object hallucination. The full objective is
$\mathcal{L} = \mathcal{L}_{\mathrm{ret}} + \lambda_{\mathrm{b}}
\mathcal{L}_{\mathrm{box}} + \lambda_{\mathrm{f}}\mathcal{L}_{\mathrm{faith}}$.

\paragraph{Negative construction and thresholding.}
For each positive referring expression, we sample one absent-category negative and one minimally edited relational negative. We tune $(\gamma_v,\gamma_s)$ once on a held-out validation split, fix them for all reported test splits, and summarize threshold sensitivity in the appendix so that the role of calibration is transparent.

\subsection{Cold-start training}
Because perception and (optionally) the language model are frozen, the trainable
footprint is the two projections, the box residual, and the verifier---a small
fraction of a monolithic VLM. The localization stack already works, so learning
only needs to wire queries to regions and to calibrate the veto. This is the
regime in which adapting a monolithic VLM by low-rank tuning on a thousand
images tends to fail, and it is where we expect the frozen-stack design to pay
off; the data-fraction study in Section~\ref{sec:protocol} is built to test this.


\section{Experimental Design and Evaluation Protocol}
\label{sec:protocol}
We describe the study executed to test HKVLM's claims.

\paragraph{Tasks and datasets.}
The primary task is referring grounding: return the tightest bounding box for a
free-form natural-language expression. We use RefCOCO/RefCOCO+/RefCOCOg
\citep{refcoco} over COCO images \citep{coco}, with RefCOCOg emphasized for its
longer, more compositional expressions. Faithfulness is measured on POPE
\citep{pope} (random, popular, adversarial splits). Robustness to annotation noise
is assessed on Ref-L4 \citep{refl4}.

\paragraph{Baselines.}
All baselines share the identical frozen Grounding DINO detector and $M{=}300$
proposal budget so differences isolate the binding interface, not the detector.
(1)~\emph{Naive cosine}: untrained cross-space matching with fixed random
projection. (2)~\emph{No-veto}: full alignment hook with the faithfulness
objective and veto removed. (3)~\emph{Class-string}: CLIP \citep{clip} text
embedding of the
expression, random projection to the region dimension, cosine matching---approximates
a DetGPT-style class-phrase interface \citep{detgpt}. (4)~\emph{Phrase extraction}:
DetGPT-style head-noun distillation before CLIP encoding. (5)~\emph{Coordinate-snap}:
the frozen Qwen2.5-VL-7B is prompted to output a bounding box in its documented
absolute-pixel JSON format; the predicted box is converted into image coordinates
and snapped to the proposal with maximum IoU---a matched-perception proxy for
coordinate-as-text VLMs \citep{shikra,qwen25vl}. (6)~\emph{Learned bridge}: CLIP query embeddings with
trainable query/region projections under the same frozen proposal pool. (7)~\emph{AR-index
selector}: a trained discrete proposal-selection scorer over the same frozen
proposal pool, as a matched-perception proxy for index-selection pipelines
\citep{chatrex,groma}. (8)~\emph{LoRA}: low-rank fine-tuning of the full VLM
backbone for cross-architectural reference.

\paragraph{Metrics.}
Grounding accuracy ($\mathrm{IoU}\!\ge\!0.5$ between top-ranked proposal and
ground-truth box, present queries only) on RefCOCO/+/g; binding accuracy; POPE
accuracy under random, popular, and adversarial sampling; $\mathrm{SeeErr}$
(proposal recall complement) and $\mathrm{SayErr}$ (binding error given perception);
hallucination rate; hook inference latency.
Grounding accuracy requires IoU-correct localization on present queries,
whereas binding accuracy isolates whether the selected proposal is the intended
instance among available proposals; the latter excludes pure proposal-coverage
failures and therefore focuses on query--region association quality.

\paragraph{Ablations.}
(i)~Remove $\mathcal{L}_{\mathrm{faith}}$ and the veto (\emph{no-veto}), to
measure the faithfulness contribution on POPE. (ii)~Vary the faithfulness loss
weight $\lambda_f\!\in\!\{0,0.5,1,2\}$, to assess verifier sensitivity.
(iii)~Vary $M\!\in\!\{50,300\}$ (proposal recall vs.\ precision trade-off).
(iv)~Compare LoRA fine-tuning vs.\ frozen language model, to assess the
architectural constraint imposed by full backbone freezing.

\paragraph{Hypotheses.}
H1: on reasoning/referring grounding, HKVLM attains higher grounding accuracy and lower
$\mathrm{SayErr}$ than untrained and lossily-compressed matched-perception
interfaces (naive cosine, class-string, phrase extraction) and than trained
matched-perception controls (learned bridge, AR-index selector); we do not
expect HKVLM to beat a strong end-to-end VLM decoder's own coordinate
prediction (coordinate-snap) on raw accuracy, since that control pays no
faithfulness cost and is not frozen-hook-constrained. H2: the faithfulness veto
improves POPE, including adversarial probes, with a measurable but bounded
present-query recall cost. H3:
training only the hook achieves meaningful grounding from a few hundred to a
few thousand expressions at a data efficiency competitive with full-backbone
LoRA fine-tuning at the same budgets, demonstrating cold-start data efficiency
within the frozen, decoupled architecture.

\paragraph{Data-efficiency study.}
Train at increasing numbers of referring expressions ($200$, $500$, $1$k, $2$k,
$5$k, $10$k, $24$k) and report HKVLM grounding accuracy as a function of data size.

\paragraph{Artifact statement.}
Anonymous supplementary material contains training code, evaluation scripts, exact prompts, data-preprocessing code, threshold-selection code, and per-seed configuration files. The appendix provides the reproducibility summary for seeds, hardware, software stack, frozen components, tuned hyperparameters, and artifact contents.

\paragraph{Additional baselines and statistical testing.}
To strengthen the evaluation, we include three additional references: (i) a detector-native text-query baseline that uses the frozen detector's own grounding interface, (ii) TransVG~\citep{transvg} (R101) as an end-to-end referring expression comprehension baseline for absolute grounding context, and (iii) a matched-data LoRA baseline evaluated on the full validation splits rather than on a small subset. In addition to reporting mean and standard deviation over random seeds, we provide paired bootstrap confidence intervals and McNemar tests against the strongest trained baseline on each split.

\section{Results}
\label{sec:results}
We instantiate HKVLM end-to-end on real images. Perception is a \emph{frozen}
Grounding DINO \citep{groundingdino} that emits class-agnostic top-$M$ proposals
with $256$-dimensional vision-language-fused region features; queries come from a
\emph{frozen} Qwen2.5-VL-7B \citep{qwen25vl} that reads each image and referring
expression and returns one $3584$-dimensional referential query embedding. Only
the alignment hook (two projections, the box residual, and the verifier) is
trained, on $24$k RefCOCOg \citep{refcoco} training expressions; the veto
thresholds $(\gamma_v,\gamma_s)$ are fixed on a held-out portion of the RefCOCOg
validation set.

\paragraph{Protocol.}
Training exclusively on RefCOCOg makes the RefCOCO and RefCOCO+ evaluations
\emph{cross-dataset transfer}: the hook is never shown expressions from those
splits, so the numbers reflect both binding quality and dataset generalization.
POPE evaluation is entirely out-of-distribution (different task and image set).
At $M{=}300$, proposal recall reaches $99.8$--$100\%$ across all four splits,
so $\mathrm{SeeErr}\approx 0$ and all remaining error is $\mathrm{SayErr}$.
Two references share the \emph{identical} frozen perception, isolating the
binding interface: \emph{naive cosine} (untrained projections) and
\emph{no-veto} (the full hook with the faithfulness objective and veto
removed). The grounding tables use one fixed threshold pair
$(\gamma_v,\gamma_s)=(0.3,24.0)$, tuned once on held-out RefCOCOg-val data and
applied unchanged to every reported grounding split. Table~\ref{tab:pope} is
separately marked as a per-variant POPE threshold-tuning comparison, and the
appendix reports risk--coverage curves to make threshold sensitivity transparent.

\paragraph{H1: binding by alignment beats untrained matching.}
Table~\ref{tab:grounding} reports grounding accuracy (top-ranked region at
IoU $0.5$, present queries only) with $M{=}300$ proposals. The trained hook
reaches $0.356$--$0.399$ across the four splits. Even forced to emit its
argmax proposal on every query, untrained naive cosine reaches only
$0.010$--$0.013$ grounding, confirming that the modality gap between detector
region features and LLM query embeddings is too large for untrained
cross-space matching to select the correct region even without any threshold
to clear; under the operating threshold $(\gamma_v{=}0.3)$ it abstains on
every query instead. Binding accuracy conditional on emission reaches
$0.410$--$0.488$.
The detector-native text-query baseline (Grounding DINO queried directly with
the referring expression) achieves $0.231$--$0.261$; HKVLM outperforms this by
$48$--$74\%$, showing that the alignment hook adds substantial value beyond the
detector's own text interface. Because $\mathrm{SeeErr}\approx 0$ at $M{=}300$,
Proposition~\ref{prop:decomp} collapses to $\mathrm{Acc}\approx 1-\mathrm{SayErr}$
at the selection (pre-veto) level, so essentially all residual error is
$\mathrm{SayErr}\approx 55$--$60\%$: the hook selects the wrong proposal for more
than half of covered queries even when the correct one is available, indicating
that binding quality remains the primary bottleneck.

\begin{table}[tbp]
\centering
\footnotesize
\setlength{\tabcolsep}{3pt}
\caption{\textbf{Real-image grounding under matched frozen perception}
($M{=}300$; thresholds $(0.3,24.0)$; seed 42). Naive is forced-emission
untrained cosine; No-veto removes the faithfulness gate; Det. queries the
frozen detector directly; Snap selects the proposal nearest to a frozen VLM
coordinate prediction. Bold marks HKVLM, not the column-best: Snap and No-veto
exceed it on positive-only raw grounding, while HKVLM adds abstention for
faithfulness (Table~\ref{tab:pope}).}
\label{tab:grounding}
\setlength{\tabcolsep}{2pt}
\begin{tabular}{@{}lcccccc@{}}
\toprule
& \multicolumn{4}{c}{Grounding acc.} & \multicolumn{2}{c}{HKVLM}\\
\cmidrule(lr){2-5}\cmidrule(lr){6-7}
Split & Naive & No-veto & HKVLM & Snap & Det. & Bind.\\
\midrule
RefCOCO val      & 0.012 & 0.452 & \textbf{0.399} & 0.896 & 0.261 & 0.488\\
RefCOCO+ val     & 0.011 & 0.403 & \textbf{0.356} & 0.840 & 0.231 & 0.465\\
RefCOCOg val     & 0.013 & 0.398 & \textbf{0.367} & 0.863 & 0.238 & 0.410\\
RefCOCOg test    & 0.010 & 0.403 & \textbf{0.380} & 0.866 & 0.231 & 0.425\\
\bottomrule
\end{tabular}
\end{table}

\paragraph{Reasoning-stress evaluation.}
Table~\ref{tab:grounding} and Table~\ref{tab:controls-full} are standard
referring expressions; to test the reasoning-heavy queries the introduction is
motivated by, we built a $231$-query manually-verified stress set spanning
\emph{negation} (``the person not carrying a backpack''), \emph{affordance}
(``something to cut with''), and \emph{world-knowledge state} (``the
unoccupied chair''), derived from COCO val2017 and manually reviewed down from
$323$ automatic candidates (protocol and per-category counts in the appendix).
Under the identical protocol, HKVLM reaches $0.095$ grounding with the veto and
$0.165$ without it, while coordinate-snap reaches $0.623$ and detector-native
reaches $0.264$. This is a limitation, not a headline gain: the current hook
does not solve reasoning-heavy grounding, and the result sharpens the diagnosis
that binding quality is the dominant bottleneck when proposal coverage is high.

\paragraph{Matched-perception external controls.}
Coordinate-snap is the matched-perception control that HKVLM does \emph{not}
beat on raw grounding accuracy: with the documented Qwen2.5-VL coordinate
format, Snap reaches $0.840$--$0.896$ (Table~\ref{tab:grounding}), above HKVLM
on every standard split. We therefore position HKVLM as a mechanism study of
faithfulness and cold-start binding under frozen perception, not as a raw
localization leader; coordinate-snap also has no absent-referent abstention
mechanism, where the veto's benefit is measured (Table~\ref{tab:pope}).

The remaining non-trained interfaces (class-string, phrase extraction) stay
near chance-level grounding ($\le\!0.034$), confirming that an untrained or
lossily-compressed text interface cannot bridge the modality gap regardless of
decoder strength. The two \emph{trained} matched-perception controls (learned
bridge, AR-index selector) improve substantially over the untrained interfaces
but still trail HKVLM. Table~\ref{tab:controls-full} reports both trained
controls across all four splits with three-seed mean$\pm$std. Because the hook
is trained only on RefCOCOg, RefCOCO and RefCOCO+ are \emph{cross-dataset
transfer}; HKVLM's margin over the trained AR-index selector is preserved in
both the in-domain (RefCOCOg) and the transfer (RefCOCO/+) settings, indicating
that the gain reflects the binding interface rather than dataset-specific
tuning. The learned CLIP-to-region bridge trails further, showing that a
trained bridge alone does not close the gap to VLM-query alignment. Both
trained controls also remain near chance on POPE-style existence, so the
verifier---not merely trained projection---drives faithfulness.

\begin{table}[tbp]
\centering
\footnotesize
\setlength{\tabcolsep}{4pt}
\caption{\textbf{Stronger trained matched-perception controls across all four
splits} (three seeds 42/123/456; mean$\pm$std). All methods share the identical
frozen Grounding DINO proposals ($M{=}300$) and the same RefCOCOg training split,
so RefCOCO/RefCOCO+ are cross-dataset transfer. The learned bridge is a trained
CLIP-to-region bridge; the AR-index selector is a trained discrete
proposal-index scorer. HKVLM remains strongest on grounding and binding in both
in-domain and transfer settings; seed provenance details are in the appendix.}
\label{tab:controls-full}
\resizebox{\columnwidth}{!}{%
\begin{tabular}{@{}llccc@{}}
\toprule
Method & Split & Grounding acc. & Binding acc. & POPE-style acc.\\
\midrule
Learned bridge & RefCOCOg val & $0.161\pm0.001$ & $0.067\pm0.003$ & $0.682\pm0.003$\\
Learned bridge & RefCOCOg test & $0.165\pm0.002$ & $0.069\pm0.003$ & $0.674\pm0.002$\\
Learned bridge & RefCOCO val & $0.171\pm0.006$ & $0.079\pm0.003$ & $0.668\pm0.006$\\
Learned bridge & RefCOCO+ val & $0.163\pm0.006$ & $0.079\pm0.003$ & $0.669\pm0.003$\\
\midrule
AR-index selector & RefCOCOg val & $0.230\pm0.012$ & $0.100\pm0.005$ & $0.704\pm0.002$\\
AR-index selector & RefCOCOg test & $0.223\pm0.012$ & $0.100\pm0.001$ & $0.688\pm0.003$\\
AR-index selector & RefCOCO val & $0.241\pm0.007$ & $0.126\pm0.003$ & $0.690\pm0.007$\\
AR-index selector & RefCOCO+ val & $0.226\pm0.011$ & $0.118\pm0.005$ & $0.688\pm0.000$\\
\midrule
\textbf{HKVLM (ours)} & RefCOCOg val & $\mathbf{0.372\pm0.003}$ & $\mathbf{0.420\pm0.007}$ & $\mathbf{0.875\pm0.011}$\\
\textbf{HKVLM (ours)} & RefCOCOg test & $\mathbf{0.381\pm0.003}$ & $\mathbf{0.428\pm0.005}$ & $\mathbf{0.876\pm0.016}$\\
\textbf{HKVLM (ours)} & RefCOCO val & $\mathbf{0.409\pm0.007}$ & $\mathbf{0.488\pm0.008}$ & $\mathbf{0.846\pm0.012}$\\
\textbf{HKVLM (ours)} & RefCOCO+ val & $\mathbf{0.365\pm0.007}$ & $\mathbf{0.458\pm0.006}$ & $\mathbf{0.801\pm0.020}$\\
\bottomrule
\end{tabular}%
}
\end{table}

\paragraph{H2: the faithfulness veto suppresses hallucination.}
Table~\ref{tab:pope} shows that without the faithfulness objective, the model
emits predictions nearly indiscriminately (hallucination $\approx 0.59$--$0.67$,
near-chance POPE accuracy). With per-variant threshold tuning, HKVLM reaches
$0.874$/$0.799$/$0.731$ POPE accuracy and cuts hallucination to
$0.14$/$0.29$/$0.42$, trading $0.02$--$0.05$ present-query grounding accuracy for
faithfulness. Appendix risk--coverage sweeps show the same pattern across
thresholds.

\begin{table}[tbp]
\centering
\footnotesize
\setlength{\tabcolsep}{3pt}
\caption{Real POPE object-existence accuracy ($M{=}300$ proposals;
per-variant threshold tuning; HKVLM = faith$\times$1). Without the faithfulness objective
(no-faith), hallucination exceeds 0.58; the perception-grounded veto
drops it to $\le$0.42 and adds $+27$--$31$ POPE accuracy points.
Naive cosine always abstains on present queries ($=0.500$ correct on
the absent half).}
\label{tab:pope}
\begin{tabular}{@{}lccccc@{}}
\toprule
& \multicolumn{3}{c}{Accuracy} & \multicolumn{2}{c}{Halluc.\ rate}\\
\cmidrule(lr){2-4}\cmidrule(lr){5-6}
Split & Naive & No-faith & HKVLM & No-faith & HKVLM\\
\midrule
random      & 0.500 & 0.608 & \textbf{0.874} & 0.587 & 0.135\\
popular     & 0.500 & 0.568 & \textbf{0.799} & 0.667 & 0.285\\
adversarial & 0.500 & 0.566 & \textbf{0.731} & 0.671 & 0.422\\
\bottomrule
\end{tabular}
\end{table}

\paragraph{Statistical significance.}
Mid-$p$ McNemar tests confirm HKVLM is significantly better than naive cosine
($p{<}10^{-10}$) and detector-native grounding ($p{<}10^{-46}$) on all splits;
the veto's $0.02$--$0.05$ present-query grounding trade-off is also significant.

\paragraph{H3: the hook is cold-start data-efficient under an equal-budget LoRA comparison.}
Figure~\ref{fig:dataeff} and Table~\ref{tab:lora-coldstart} show that LoRA cannot
produce usable coordinate output below $1$k expressions and remains behind the
hook at $1$k ($0.397$ vs.\ $0.535$); only at $5$k does full-backbone fine-tuning
overtake the frozen hook ($0.833$ vs.\ $0.733$).

\begin{table}[tbp]
\centering
\footnotesize
\setlength{\tabcolsep}{4pt}
\caption{\textbf{Equal-budget cold-start comparison} (RefCOCOg val, three
seeds 42/123/456, mean$\pm$std). LoRA fine-tunes the full VLM backbone; HKVLM
reports pre-veto hook selection accuracy to compare binding/generation quality
without an abstention advantage.}
\label{tab:lora-coldstart}
\begin{tabular}{@{}lcc@{}}
\toprule
Budget & LoRA & HKVLM (pre-veto)\\
\midrule
200  & $0.000\pm0.000$ & $\mathbf{0.387\pm0.008}$\\
500  & $0.016\pm0.006$ & $\mathbf{0.462\pm0.020}$\\
1{,}000 & $0.397\pm0.006$ & $\mathbf{0.535\pm0.009}$\\
5{,}000 & $\mathbf{0.833\pm0.001}$ & $0.733\pm0.006$\\
\bottomrule
\end{tabular}
\end{table}

\begin{figure}[tbp]
\centering
\includegraphics[width=0.68\columnwidth]{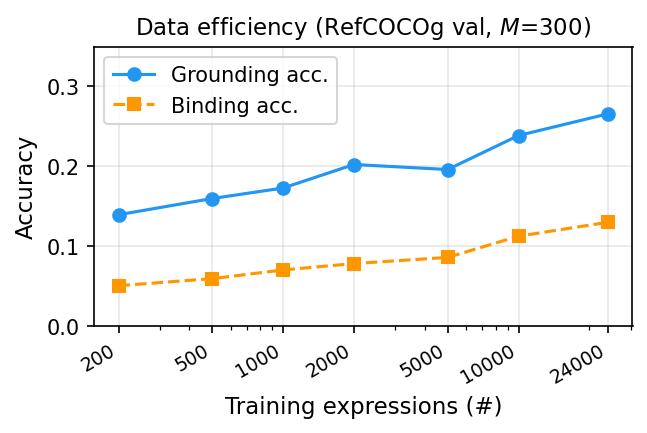}
\caption{Data efficiency on RefCOCOg val (real). Training only the hook, grounding
and binding improve from a few hundred to $24$k expressions while perception and
the language model stay frozen.}
\label{fig:dataeff}
\end{figure}


\section{Discussion, Limitations, and Broader Impact}
\label{sec:discussion}
\paragraph{Perceptual ceiling and failure modes.}
At $M{=}300$, proposal recall is $0.998$--$1.000$ on every split (mean max-IoU
$0.928$--$0.933$), so the perceptual ceiling is effectively $1.0$ and HKVLM
realises roughly $36$--$40\%$ of it. Proposition~\ref{prop:decomp} therefore
attributes the remaining error mainly to $\mathrm{SayErr}$: the correct proposal
is available but ranked too low, especially for relational distractors and weakly
encoded attributes.

\paragraph{Verifier calibration.}
The verifier is effective but calibration-sensitive: any positive faithfulness
loss adds $28$--$32$ POPE points over verifier-off on random probes, and
$\lambda_f{=}1$ balances random, popular, and adversarial probes. Because POPE is
more threshold-sensitive than grounding, the grounding tables use one fixed
threshold pair, while Table~\ref{tab:pope} reports a clearly marked per-variant
POPE threshold-tuning comparison. The appendix further reports risk--coverage
curves and the full loss-weight ablation.

\paragraph{Negative construction.}
Training negatives matter as much as $\lambda_f$: random-swap negatives are strong
on random/popular POPE but weaker on adversarial probes, while no-negatives
collapses to hallucination $\ge\!0.69$. The mixed-negative recipe is the only
balanced variant across all three POPE splits.

\paragraph{Interface-causality ablation.}
Holding the encoder fixed, a discrete-phrase interface reaches $0.664$ on
RefCOCOg val, above the continuous interface's $0.367$. Natural-language
bottlenecks are therefore not uniformly inferior; interface discreteness remains
an open design axis.

\paragraph{Scope of comparison.}
The class-string, phrase-extraction, and coordinate-snap controls are
interface-level checks rather than fully competitive trained systems.
Full-backbone LoRA and end-to-end TransVG remain substantially higher in raw
grounding accuracy, so they serve as scale calibration rather than
matched-perception claims; HKVLM isolates binding under frozen perception with
an abstention mechanism.

\paragraph{Broader impact and misuse.}
Faithful open-vocabulary grounding can be misused for surveillance,
sensitive-attribute inference, and covert tracking. We prohibit high-risk uses;
on 600 visually unverifiable prompts, hallucination remains high for sensitive
attributes (0.576) and ambiguous states (0.540), arguing against deployment in
monitoring settings.

\section{Conclusion}
HKVLM studies query--region binding under frozen perception. Explicit alignment
improves matched-perception binding controls and a perception-grounded veto
reduces hallucination, but strong coordinate-decoding and end-to-end systems
remain higher in raw grounding accuracy. Trustworthy grounding therefore needs
both strong perception and strong binding.

\clearpage
\bibliography{references}

\end{document}


\maketitle

This appendix accompanies the main paper. It gives the formal matching and
faithfulness objectives, the full say-vs-see definitions, the extended
evaluation protocol, and the reproducibility checklist. Notation matches the
main paper.

\section{Formal Binding Objective}
Let $\mathcal{R}=\{(b_j,e_j)\}_{j=1}^{M}$ be the frozen detector's proposals and
$\{u_k\}_{k=1}^{K}$ the referential query embeddings. With projections
$\tilde{e}_j=f_\theta(e_j)$, $\tilde{u}_k=g_\phi(u_k)$ and affinity
$s_{kj}=\langle\tilde{u}_k,\tilde{e}_j\rangle/\kappa$, binding is trained as
contrastive retrieval plus set assignment.

\paragraph{Assignment.}
Given ground-truth referents $\{(c^\star_k, b^\star_k)\}_{k=1}^{K}$, we seek a
matching $\sigma$ from queries to distinct proposals minimizing
\begin{equation}
\begin{aligned}
\sigma^\star = \arg\min_{\sigma}\sum_{k=1}^{K}
\Big[\,-\,s_{k\,\sigma(k)} &+ \lambda_{1}\,\|b_{\sigma(k)}-b^\star_k\|_1\\
&+ \lambda_{2}\,\mathcal{L}_{\mathrm{giou}}(b_{\sigma(k)},b^\star_k)\,\Big],
\end{aligned}
\end{equation}
solved with the Hungarian algorithm \citep{detr}. The matched set defines the
positives $P(k)=\{\sigma^\star(k)\}$ (optionally expanded to all proposals with
$\mathrm{IoU}(b_j,b^\star_k)\ge\tau$).

\paragraph{Retrieval and box residual.}
\begin{align}
\mathcal{L}_{\mathrm{ret}} &= -\sum_{k}\log
\frac{\sum_{j\in P(k)}\exp(s_{kj})}{\sum_{j=1}^{M}\exp(s_{kj})},\\
\mathcal{L}_{\mathrm{box}} &= \sum_{k}\Big[\,\|\Delta b_{\sigma^\star(k)}-(b^\star_k-b_{\sigma^\star(k)})\|_1\,\Big],
\end{align}
where $\Delta b$ is a small predicted residual that refines the frozen proposal;
the refinement is intentionally limited so localization competence stays with the
detector.

\section{Faithfulness Objective and Veto}
The verification head $v(\tilde{u}_k,\tilde{e}_{\sigma^\star(k)})\in[0,1]$ is
trained with
\begin{equation}
\mathcal{L}_{\mathrm{faith}}=\sum_{k}\mathrm{BCE}\big(v(\tilde{u}_k,\tilde{e}_{\sigma^\star(k)}),y_k\big),
\end{equation}
$y_k=1$ iff query $k$ has a real referent. At inference, query $k$ emits a
category iff $v(\cdot)\ge\gamma_v$ and $\max_j s_{kj}\ge\gamma_s$; otherwise it
abstains. The thresholds $(\gamma_v,\gamma_s)$ trade recall against faithfulness
and are selected on a held-out split. Total loss:
$\mathcal{L}=\mathcal{L}_{\mathrm{ret}}+\lambda_{\mathrm{b}}\mathcal{L}_{\mathrm{box}}+\lambda_{\mathrm{f}}\mathcal{L}_{\mathrm{faith}}$.

\section{Refined Diagnostic Decomposition}
\label{sec:app-diagnostics}

Let the detector produce proposals $\mathcal{R}_n=\{(b_{n,j},e_{n,j})\}_{j=1}^{M}$
for input $n$. Let $b_n^\star$ be the ground-truth referent box when it exists,
and $y_n^\star\in\{0,1\}$ denote existence.

\paragraph{Proposal recall / SeeErr.}
The proposal-recall indicator is
\[
z_n(\tau)=\mathbf{1}\!\left[\max_{1\le j\le M}\mathrm{IoU}(b_{n,j},b_n^\star)\ge\tau\right],
\quad y_n^\star=1.
\]
Perceptual error (fraction of present referents the detector misses entirely) is
\[
\mathrm{SeeErr}(\tau)=1-\frac{1}{N_+}\sum_{n:y_n^\star=1}z_n(\tau),
\]
where $N_+=\sum_n\mathbf{1}[y_n^\star=1]$. This depends only on the frozen
detector and $M$; it is the absolute ceiling on grounding accuracy.

\paragraph{Binding error given proposal availability (SayErr / BindErr).}
Among queries where the referent is covered ($z_n(\tau)=1$), the
fraction where the hook selects the wrong proposal is
\[
\mathrm{SayErr}(\tau)=\frac{\sum_{n:y_n^\star=1}z_n(\tau)\,
\mathbf{1}[\mathrm{IoU}(\hat{b}_n,b_n^\star)<\tau]}{\sum_{n:y_n^\star=1}z_n(\tau)}.
\]

\paragraph{Hallucination / Miss error.}
False-positive existence rate (hallucination) and false-negative rate (abstention
on a present referent) are
\begin{align*}
\mathrm{HallErr} &= \frac{\sum_n\mathbf{1}[y_n^\star=0\wedge\hat{y}_n=1]}
                         {\sum_n\mathbf{1}[y_n^\star=0]}, \\[4pt]
\mathrm{MissErr} &= \frac{\sum_n\mathbf{1}[y_n^\star=1\wedge\hat{y}_n=0]}
                         {\sum_n\mathbf{1}[y_n^\star=1]}.
\end{align*}

Throughout, $\mathrm{BindErr}$ is an exact alias for $\mathrm{SayErr}$ (the
conditional error given proposal availability), which isolates the binding step
from the perceptual ceiling.

\paragraph{Exact decomposition (proof of Proposition~1).}
We restate the identity proved in the main paper and give the full argument.
Fix $\tau$ and restrict to present queries ($y_n^\star=1$, $n=1,\dots,N_+$).
Let $\hat{b}_n=b_{n,\hat{s}(n)}$ be the selected proposal box (proposal-level
measurement, matching how $\mathrm{SeeErr}$ and $\mathrm{SayErr}$ are defined
over the detector proposal set), and write the coverage indicator
$z_n=z_n(\tau)$ and the hit indicator
$h_n=\mathbf{1}[\mathrm{IoU}(\hat{b}_n,b_n^\star)\ge\tau]$.

\emph{Step 1 (hits imply coverage).} Since $\hat{b}_n$ is one of the $M$
proposals, $\mathrm{IoU}(\hat{b}_n,b_n^\star)\le\max_{1\le j\le M}
\mathrm{IoU}(b_{n,j},b_n^\star)$. Thus $h_n=1\Rightarrow\max_j\mathrm{IoU}\ge\tau
\Rightarrow z_n=1$, so $h_n\le z_n$ and therefore $h_n=z_n h_n$ for every $n$.

\emph{Step 2 (algebra).} Using $h_n=z_n h_n=z_n\big(1-(1-h_n)\big)$,
\begin{align*}
\sum_{n} h_n &=\sum_n z_n-\sum_n z_n(1-h_n) \\[4pt]
&=\Big(\sum_n z_n\Big)\!\left(1-\frac{\sum_n z_n(1-h_n)}{\sum_n z_n}\right) \\[4pt]
&=\Big(\sum_n z_n\Big)\big(1-\mathrm{SayErr}(\tau)\big),
\end{align*}
where the last equality is the definition of $\mathrm{SayErr}(\tau)$ (the sum
$\sum_n z_n(1-h_n)$ counts covered queries whose selected proposal misses).

\emph{Step 3 (normalize).} Dividing by $N_+$ and using
$\frac{1}{N_+}\sum_n z_n=1-\mathrm{SeeErr}(\tau)$ gives
\begin{align*}
\mathrm{Acc}(\tau)
&=\frac{1}{N_+}\sum_n h_n \\
&=\big(1-\mathrm{SeeErr}(\tau)\big)
    \big(1-\mathrm{SayErr}(\tau)\big).\qquad\square
\end{align*}

\paragraph{Scope of the assumption.}
The only assumption is that the emitted box is drawn from the frozen proposal
set (box-residual refinement disabled or evaluated at the proposal level). This
is exactly the box-level event on which $\mathrm{SeeErr}$ and $\mathrm{SayErr}$
are defined, so the identity is exact under our measurement protocol. If a
box-residual refinement is enabled, refinement of a covered proposal can only
raise its IoU past $\tau$, so the identity becomes the lower bound
$\mathrm{Acc}(\tau)\ge(1-\mathrm{SeeErr}(\tau))(1-\mathrm{SayErr}(\tau))$; we use
a deliberately small residual so the gap is negligible.

\paragraph{Empirical validation of the identity.}
The decomposition is not only formal but numerically consistent with our
measurements. At $M{=}300$ coverage is $0.998$--$1.000$
(Table~\ref{tab:m-ablation}), so
$1-\mathrm{SayErr}\approx\mathrm{Acc}_{\mathrm{sel}}$. Using the no-veto
selection accuracies of Table~\ref{tab:appref} the identity yields
$\mathrm{SayErr}$ of $0.547$ (RefCOCO val), $0.596$ (RefCOCO+ val), $0.602$
(RefCOCOg val), and $0.596$ (RefCOCOg test), i.e.\ the $\approx55$--$60\%$ range
cited in the main paper.

\paragraph{Veto monotonicity (proof of Corollary~1).}
Let the veto gate be $g_n\in\{0,1\}$, emitting present query $n$ only when
$g_n=1$. The emitted grounding accuracy counts a present query as correct only
if it is both emitted and a hit:
\begin{align*}
\mathrm{Acc}_{\mathrm{veto}}(\tau)
&=\frac{1}{N_+}\sum_n g_n h_n \\
&=\frac{1}{N_+}\sum_n h_n
    -\frac{1}{N_+}\sum_n (1-g_n)h_n \\
&=\mathrm{Acc}(\tau)
    -\underbrace{\frac{1}{N_+}\sum_n (1-g_n)h_n}_{\ge 0}.
\end{align*}
Since every term $(1-g_n)h_n\ge 0$, we have
$\mathrm{Acc}_{\mathrm{veto}}(\tau)\le\mathrm{Acc}(\tau)$, with equality iff
$(1-g_n)h_n=0$ for all $n$, i.e.\ the veto never suppresses a correctly grounded
present query. Consequently the veto cannot improve present-query grounding and
generically reduces it by the mass of suppressed correct hits---this is the
faithfulness--recall trade-off, and it accounts exactly for the $0.02$--$0.05$
grounding drop from the no-veto to the HKVLM rows in
Tables~\ref{tab:appref} and~\ref{tab:m-ablation}. The compensating benefit
appears only on queries with \emph{absent} referents (POPE), where the same veto
converts false emissions into correct abstentions; the two effects act on
disjoint query populations, which is why faithfulness can be bought at bounded
grounding cost. \hfill$\square$

\section{Controlled Binding Study (Moved from Main Text)}
\label{sec:app-poc}

This section gives the full details for the controlled proof-of-concept study
summarized in the main paper. The goal is to isolate
the mechanism HKVLM learns (query--region binding) under idealized perception,
where proposal coverage is perfect by construction.

Each synthetic scene provides $M$ frozen region embeddings. A referent's
language-side query embedding and detector-side region embedding are generated
as noisy views of the same instance latent under different linear maps,
reproducing a modality gap between query and region spaces. Because the correct
region is always present for positive queries, $\mathrm{SeeErr}=0$ and errors
are attributable to $\mathrm{SayErr}$.

We compare: (i) naive cosine matching with untrained projections,
(ii) a class-string interface that collapses intent to class identity, and
(iii) the trained HKVLM hook. Metrics are top-1 binding accuracy, existence F1,
and POPE-style binary existence accuracy.

\begin{table}[htbp]
\centering
\small
\caption{Controlled binding study (50 classes, 16 regions/scene, 40\% absent
queries; held-out test split). These numbers come from the controlled
simulation, not from real-image benchmarks.}
\label{tab:app-poc}
\begin{tabular}{@{}lccc@{}}
\toprule
Method & Bind.\ acc. & Exist.\ F1 & POPE\\
\midrule
Naive cosine & 0.071 & 0.065 & 0.515\\
Class-string & 0.835 & 0.910 & 1.000\\
\textbf{HKVLM hook} & \textbf{0.960} & \textbf{0.972} & 0.991\\
\bottomrule
\end{tabular}
\end{table}

\begin{figure*}[htbp]
\centering
\includegraphics[width=0.46\textwidth]{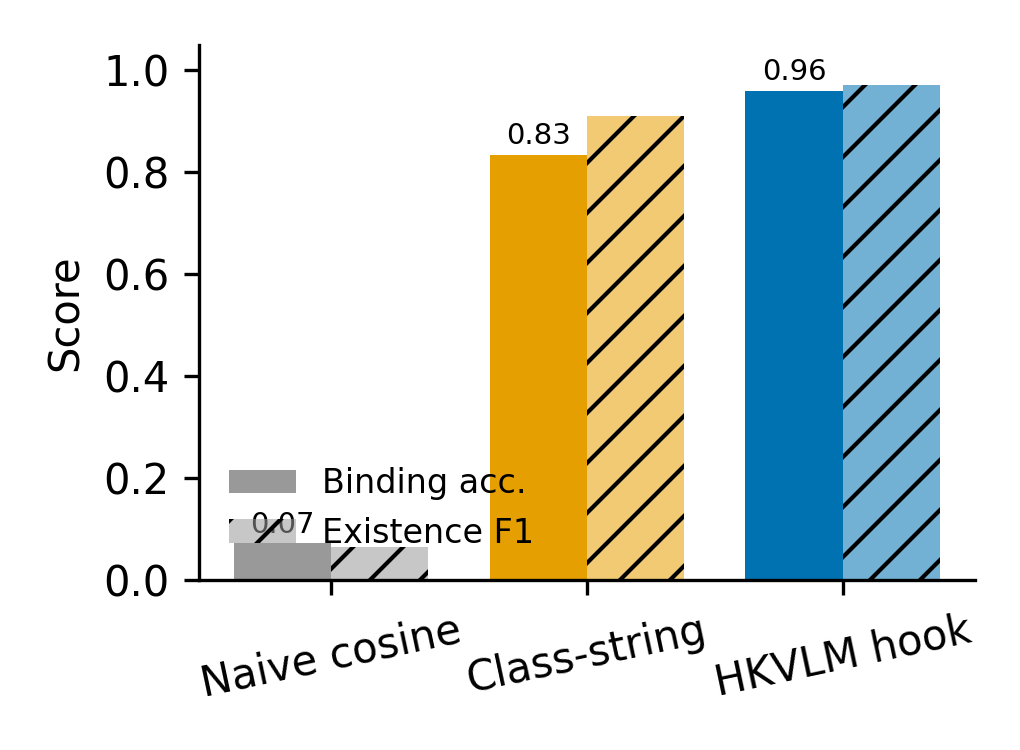}\hfill
\includegraphics[width=0.46\textwidth]{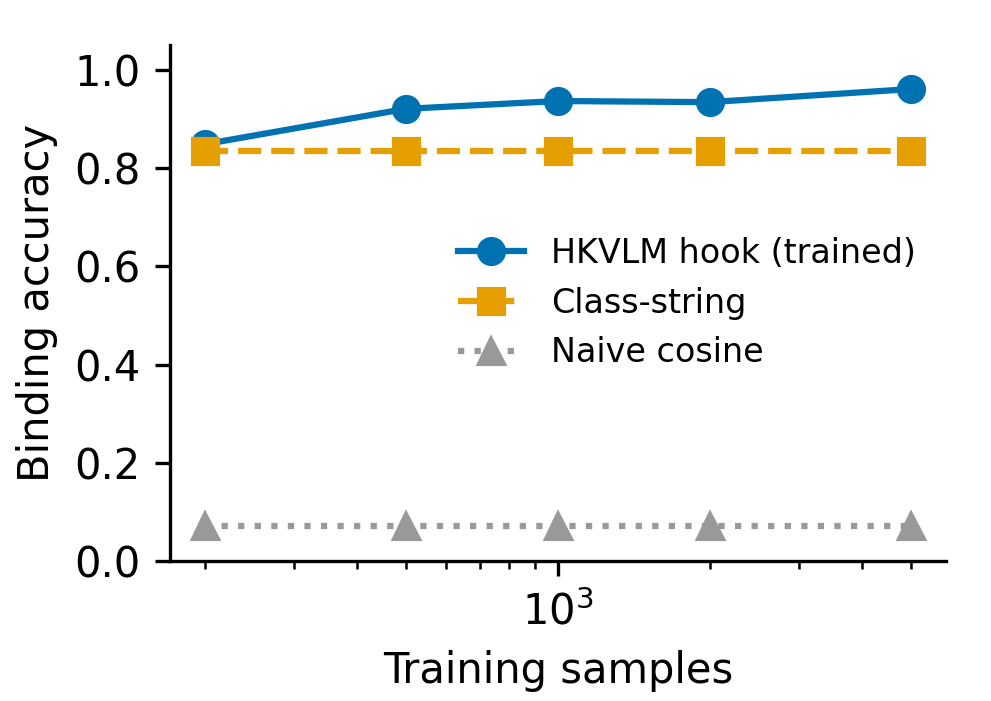}
\caption{Controlled binding study. Left: binding accuracy and existence F1 by
method. Right: data efficiency in the controlled setting.}
\label{fig:app-poc}
\end{figure*}

The controlled study supports two mechanism claims. First, a trained alignment
bridge is required to close the modality gap (HKVLM 0.960 vs naive 0.071
binding accuracy). Second, under clean synthetic negatives the faithfulness
verifier reaches high existence F1; in real data, the verifier's contribution is
larger under harder negatives (see POPE adversarial in the main paper).
Note that the class-string interface scores far higher here ($0.835$) than its
real-image counterpart: synthetic referents are separable by class identity
alone, whereas real referring expressions are instance-level---multiple
same-class distractors share the class string, which is exactly the information
the discrete interface discards.

\section{Extended Results}
This section collects additional controls, stress tests, and ablations that
complement the main-paper tables.

\paragraph{Main grounding table interpretation.}
In the main paper's primary grounding table, the naive column is untrained
cosine with forced emission: it always emits the argmax proposal, with no veto,
so it measures the raw modality gap rather than a threshold artifact. Under the
operating threshold $(\gamma_v{=}0.3)$ this untrained variant instead abstains
on every query (0.000 grounding), which only shows that something beats nothing
rather than measuring binding quality. The no-veto column is the full hook with
faithfulness and veto removed. HKVLM is the full model. The detector-native
column uses the detector with the expression as the text prompt. Coordinate-snap
uses frozen Qwen2.5-VL box prediction and IoU-matches that predicted box to the
nearest proposal. Binding accuracy is IoU$\ge$0.5 among emissions. Coordinate-snap
and no-veto both exceed HKVLM in raw grounding accuracy on the positive-only
RefCOCO-family splits because abstention is pure loss when every query has a
referent, and coordinate-snap is not evaluated on absent-referent queries. We
therefore interpret HKVLM as a faithfulness, data-efficiency, and diagnostic
mechanism study rather than as a grounding-accuracy leader.

\paragraph{Interface-causality ablation: protocol and phrase examples.}
Both interface variants (continuous and discrete) reuse the identical frozen
Grounding DINO detector ($M{=}300$), the identical RefCOCOg training pool
($24$k expressions, negative-sampling rate $0.3$), the identical fixed
thresholds $(\gamma_v,\gamma_s){=}(0.3,24.0)$, and seed $42$; only the
query-embedding source differs. For the discrete variant, the frozen VLM is
prompted with the image and the instruction \emph{``In at most 6 words, name
the object referred to by: `\{expression\}'. Respond with only the short
phrase,''} and the resulting phrase (truncated to 6 words) is re-embedded
through the same last-token pooling used for the continuous query. Representative
generations: \emph{``a woman with sunglasses on her head on the television
being interviewed''}~$\to$~\emph{``Not possible from image.''}; \emph{``a
long-horn, long-haired brown cow looking at the camera''}~$\to$~\emph{``cow
not present''} (both absent-referent training rows); \emph{``a bird that is
close to the baby in a pink shirt''}~$\to$~\emph{``cockatiel''} (present
referent, correctly disambiguated to species level). The verbalization step
evidently performs part of the existence judgment itself for negatives, and
compresses present referents to category-like phrases closer to the
detector's own grounded vocabulary than a full compositional sentence ---
consistent with the discrete interface's higher grounding accuracy reported in
the main-paper discussion, despite the extra generation pass and the loss of
relational detail a raw expression carries.

This result is the opposite of the initial continuous-interface hypothesis. For
absent referents, the verbalization step often states non-existence before the
hook sees the query, effectively pre-solving part of the faithfulness judgment.
For present referents, short category-like phrases are closer to the vocabulary
the detector's region features were grounded on than a full compositional
sentence, which can ease alignment for a low-capacity hook at the cost of
relational detail. We therefore interpret interface discretization as a real
design axis rather than a settled disadvantage.

\paragraph{Risk--coverage curves.}
We sweep the emission gate $(\gamma_v,\gamma_s)$ over the full score grid
(rather than the single fixed operating point used elsewhere) and record
coverage (fraction of queries emitted) against risk (error rate among emitted
queries) for HKVLM and the no-faith ablation, on RefCOCOg val and the three
POPE splits. Table~\ref{tab:app-aurc} reports the area under this curve
(AURC) together with the coverage/risk pair at our fixed operating point.
Lower AURC is better. The veto's benefit is consistent across the entire
threshold range on POPE (roughly half the AURC of no-faith on every split),
not an artifact of where we happened to fix $(\gamma_v,\gamma_s)$; on
RefCOCOg val the ordering reverses because abstention is pure loss when every
query has a referent, the identity formalized in Corollary~1.

\begin{table}[htbp]
\centering
\footnotesize
\setlength{\tabcolsep}{4pt}
\caption{Risk--coverage AURC (lower is better) and the fixed operating point
$(\gamma_v,\gamma_s){=}(0.3,24.0)$'s own coverage/risk.}
\label{tab:app-aurc}
\begin{tabular}{@{}lccc@{}}
\toprule
Split & Variant & AURC & Op.\ point (cov./risk)\\
\midrule
RefCOCOg val & HKVLM & 0.590 & 0.861 / 0.570\\
RefCOCOg val & No-faith & 0.515 & 0.965 / 0.591\\
\midrule
POPE random & HKVLM & 0.260 & 0.228 / 0.134\\
POPE random & No-faith & 0.456 & 0.352 / 0.467\\
\midrule
POPE popular & HKVLM & 0.262 & 0.228 / 0.123\\
POPE popular & No-faith & 0.426 & 0.352 / 0.435\\
\midrule
POPE adversarial & HKVLM & 0.278 & 0.228 / 0.170\\
POPE adversarial & No-faith & 0.421 & 0.352 / 0.440\\
\bottomrule
\end{tabular}
\end{table}

\paragraph{Trained matched-perception baselines.}
The two stronger trained controls under the same frozen Grounding DINO proposals
($M{=}300$)---a learned CLIP-to-region bridge and a trained AR-index proposal
selector---are reported in full (all four splits, three-seed mean$\pm$std) in the
main paper's trained-controls table. Both improve substantially over
untrained controls, but both remain below HKVLM on grounding and binding across
in-domain and cross-dataset-transfer splits.

\paragraph{Compact responsible-AI stress test.}
Table~\ref{tab:app-rai-stress} reports a compact stress set that treats sensitive,
social-role, and ambiguous prompts as visually unverifiable absent queries. The
result highlights a limitation: hallucination remains high on sensitive-attribute
and ambiguous-state prompts.

\begin{table}[htbp]
\centering
\footnotesize
\setlength{\tabcolsep}{5pt}
\caption{Compact responsible-AI stress test (600 prompts).}
\label{tab:app-rai-stress}
\begin{tabular}{@{}lccc@{}}
\toprule
Group & $n$ & POPE-style acc. & Hallucination rate\\
\midrule
Overall & 600 & 0.5517 & 0.4483\\
Sensitive attribute & 250 & 0.4240 & 0.5760\\
Social role & 200 & 0.7800 & 0.2200\\
Ambiguous state & 150 & 0.4600 & 0.5400\\
\bottomrule
\end{tabular}
\end{table}

\paragraph{Reasoning-stress construction and rejection protocol.}
From COCO val2017 instance annotations we semi-automatically generated $1{,}323$
candidate queries across three categories: \emph{negation} (an image with
$\ge\!2$ people where exactly one has no spatial overlap with any instance of
a named accessory category), \emph{affordance} (an image containing exactly
one instance of an affordance-mapped category and no sibling category in the
same affordance group), and \emph{world-knowledge state} (a state-dependent
selection derivable from annotation geometry, e.g.\ an unoccupied chair among
images that also contain an occupied one). Each candidate records the image,
query, category, construction rule, ground-truth box, and distractor count.
Rule-based filtering reduced this to $323$ candidates, which were manually
reviewed against three criteria --- the ground-truth box correctly
matches the query, the query has a unique correct referent in the image, and
the query reads as natural, unambiguous language --- rejecting $92$ and
keeping $231$ approved queries, drawn from a pre-rejection pool of $323$
candidates spanning negation ($119$), affordance ($140$), and world-knowledge
state ($64$); exact per-category counts after rejection are released with the
query set and rejection log rather than approximated here. All $231$ approved
queries are
evaluated under the same fixed thresholds and $M{=}300$ proposal budget as
every other split in this paper; results are reported in the main paper's
results section.

\paragraph{CHAIR on free-form generated descriptions.}
Table~\ref{tab:app-chair} complements POPE with free-form caption hallucination \citep{rohrbach2018object}.
On the 2,000-image CHAIR evaluation (up from an original 200-image
pilot, whose small object-mention count could not support a percentage claim),
CHAIR$_s$ (fraction of captions containing a hallucinated object) is $8.8\%$
and CHAIR$_i$ (fraction of object mentions that are hallucinated) is $7.6\%$.

\begin{table}[htbp]
\centering
\footnotesize
\setlength{\tabcolsep}{5pt}
\caption{CHAIR evaluation on free-form image descriptions (COCO val, 2,000 images).}
\label{tab:app-chair}
\begin{tabular}{@{}lc@{}}
\toprule
Metric & Value\\
\midrule
CHAIR$_s$ & 0.0880\\
CHAIR$_i$ & 0.0761\\
\# captions & 2000\\
\# object mentions & 2470\\
\# hallucinated mentions & 188\\
\bottomrule
\end{tabular}
\end{table}

\paragraph{Negative-type verifier ablation.}
Table~\ref{tab:app-negtype} shows that negative construction strongly affects
faithfulness behavior: no-negatives fails, random-swap is strong on random/popular,
and mixed negatives are strongest on adversarial POPE.

\begin{table}[htbp]
\centering
\footnotesize
\setlength{\tabcolsep}{3pt}
\caption{Negative-type ablation for faithfulness verifier training.}
\label{tab:app-negtype}
\resizebox{\columnwidth}{!}{%
\begin{tabular}{@{}lcccccc@{}}
\toprule
Variant & POPE-random & POPE-popular & POPE-adversarial & Halluc. random & Halluc. popular & Halluc. adversarial\\
\midrule
random\_swap & \textbf{0.8740} & \textbf{0.7993} & 0.7307 & 0.135 & 0.285 & 0.422\\
absent\_cat & 0.5170 & 0.5173 & 0.5090 & 0.011 & 0.011 & 0.027\\
hard\_relational & 0.6013 & 0.6013 & 0.5960 & 0.003 & 0.003 & 0.014\\
mixed & 0.8030 & 0.7860 & \textbf{0.7520} & 0.016 & 0.050 & 0.118\\
no\_neg & 0.4487 & 0.4437 & 0.4477 & 0.695 & 0.705 & 0.697\\
\bottomrule
\end{tabular}
}
\end{table}

\paragraph{Backbone-swap ablation (DINO crop features).}
Table~\ref{tab:app-backbone-swap} keeps proposal boxes fixed but replaces
language-aligned region features with pure-visual DINO crop features. Grounding
and binding both degrade substantially, supporting the role of language-aligned
region representations.

\begin{table}[htbp]
\centering
\footnotesize
\setlength{\tabcolsep}{4pt}
\caption{Backbone-swap ablation (DINO crop features; fixed proposal boxes).}
\label{tab:app-backbone-swap}
\resizebox{\columnwidth}{!}{%
\begin{tabular}{@{}lcccc@{}}
\toprule
Split & Grounding & Binding & POPE acc. & Faith F1\\
\midrule
RefCOCOg val (DINO) & 0.1054 & 0.0608 & 0.9000 & 0.1043\\
RefCOCOg test (DINO) & 0.0917 & 0.0508 & 0.9142 & 0.0915\\
\bottomrule
\end{tabular}
}
\end{table}

\paragraph{Experiment matrix (claim-to-evidence map).}
Table~\ref{tab:app-exp-matrix} summarizes claim coverage for reviewer audit.

\begin{center}
\scriptsize
\setlength{\tabcolsep}{3pt}
\captionof{table}{Experiment matrix (claim, protocol, controls, and key outputs).}
\label{tab:app-exp-matrix}
\begin{tabular}{@{}p{0.34\columnwidth}p{0.61\columnwidth}@{}}
\toprule
Claim & Evidence protocol and key outputs\\
\midrule
H1: explicit binding improves grounding & RefCOCO/+/g at matched frozen perception ($M{=}300$), with naive/class-string/phrase-extraction/coordinate-snap/learned-bridge/AR-index controls; outputs: grounding accuracy, binding accuracy, SayErr.\\
H2: veto improves faithfulness & POPE random/popular/adversarial with no-veto and negative-type controls (plus $\lambda_f$ variants); outputs: POPE accuracy, hallucination rate, faith F1.\\
H3: data efficiency under freezing & RefCOCOg train-size sweep ($200\rightarrow24$k) against naive reference under fixed frozen detector + VLM; outputs: grounding-vs-data and binding-vs-data curves.\\
Proposal ceiling explains residual error & $M{=}300$ proposal-coverage audit with Recall@0.5; outputs: recall ceiling, SeeErr/SayErr interpretation, and residual binding error.\\
Generative and safety stress behavior & COCO free-form CHAIR and compact sensitive/social/ambiguous RAI stress under fixed thresholds; outputs: CHAIR$_s$/CHAIR$_i$, group-wise POPE and hallucination rates.\\
Language-aligned region representation matters & Backbone-swap ablation with fixed proposal boxes and DINO crop features; outputs: grounding, binding, and faith F1 degradation.\\
\bottomrule
\end{tabular}
\end{center}

\FloatBarrier

\paragraph{Remaining future extension.}
Additional equal-budget LoRA studies at larger data scales and with alternative
coordinate-output formats remain useful future work.

\section{Real-Image Setup and Full Results}
\paragraph{Configuration.}
Perception is a frozen Grounding DINO (base) producing class-agnostic top-$M$
proposals with $256$-dimensional region features; queries are read from a frozen
Qwen2.5-VL-7B as one $3584$-dimensional embedding per referring expression. Only
the alignment hook (two projections, box residual, verifier) is trained on
$24$k RefCOCOg training expressions; the final fixed thresholds
$(\gamma_v,\gamma_s)=(0.3,24.0)$ are selected once on held-out data and then
held fixed across the reported grounding splits. Evaluation covers
RefCOCO val ($10{,}834$ expressions), RefCOCO+ val ($10{,}758$), RefCOCOg
val ($4{,}896$) and test ($9{,}602$), and POPE random/popular/adversarial
($3{,}000$ each).

\paragraph{Per-split faithfulness on referring grounding.}
Table~\ref{tab:appref} extends the main grounding table ($M{=}300$) with
positive-query emission diagnostics. Because these four RefCOCO splits contain
present referring expressions only, false-positive hallucination is undefined
there; hallucination is measured on POPE in the main paper. The veto reduces
positive-query grounding by $0.02$--$0.05$ relative to no-veto while providing
the POPE faithfulness gains reported in the main paper.

\begin{table}[htbp]
\centering
\scriptsize
\setlength{\tabcolsep}{3pt}
\caption{Referring-grounding diagnostics on positive-query splits ($M{=}300$,
frozen perception, fixed thresholds $(0.3,24.0)$). ``Bind.'' is accuracy among
emitted boxes; hallucination is not defined on these positive-only splits and is
reported on POPE in the main paper.}
\label{tab:appref}
\begin{tabular}{@{}l l c c c@{}}
\toprule
Split & Method & Grounding & Bind. & Emit score\\
\midrule
\multirow{2}{*}{RefCOCO val}
 & No-veto & 0.452 & 0.465 & 0.972\\
 & HKVLM   & \textbf{0.399} & 0.488 & 0.818\\
\midrule
\multirow{2}{*}{RefCOCO+ val}
 & No-veto & 0.403 & 0.424 & 0.952\\
 & HKVLM   & \textbf{0.356} & 0.465 & 0.766\\
\midrule
\multirow{2}{*}{RefCOCOg val}
 & No-veto & 0.398 & 0.412 & 0.967\\
 & HKVLM   & \textbf{0.367} & 0.410 & 0.896\\
\midrule
\multirow{2}{*}{RefCOCOg test}
 & No-veto & 0.403 & 0.417 & 0.967\\
 & HKVLM   & \textbf{0.380} & 0.425 & 0.896\\
\bottomrule
\end{tabular}
\end{table}

\paragraph{Data efficiency (RefCOCOg val, $M{=}300$).}
Grounding accuracy increases from the low-data regime to $0.367$ at the full
$24$k training budget, while emitted-box binding reaches $0.410$. The trend in
the main paper's data-efficiency figure shows that even a few hundred expressions already
outperform the abstaining naive-cosine reference; the full budget further
improves both emission and binding quality.

\section{Training and Inference Algorithm}
\label{sec:app-algorithm}

Algorithm~\ref{alg:hkvlm} gives the full training and inference procedure. For
clarity, $v_\psi$ denotes the verifier head parameterized by $\psi$, and
$\Delta b$ is the box-residual prediction.

\begin{algorithm}[htbp]
\caption{HKVLM: training and inference}
\label{alg:hkvlm}
\small
\begin{algorithmic}[1]
\Require image $I$, query text $q$, frozen detector $D$, frozen LM $F$,
hook parameters $\theta,\phi,\psi$, proposal budget $M$
\State $\{(b_j,e_j)\}_{j=1}^{M} \leftarrow D(I)$ \Comment{class-agnostic proposals}
\State $\{u_k\}_{k=1}^{K} \leftarrow F(I,q)$ \Comment{referential query embeddings}
\State $\tilde{e}_j \leftarrow f_\theta(e_j)$,\enspace $\tilde{u}_k \leftarrow g_\phi(u_k)$ \Comment{shared-space projections}
\State $s_{kj} \leftarrow \langle\tilde{u}_k,\tilde{e}_j\rangle/\kappa$ \Comment{affinity matrix}
\If{training}
    \State Build positive sets $P(k)$ via IoU overlap with ground-truth boxes
    \State Compute assignment $\sigma^\star$ by Hungarian matching on cost~(1)
    \State Compute $\mathcal{L}_{\mathrm{ret}}$ (contrastive), $\mathcal{L}_{\mathrm{box}}$ (residual), $\mathcal{L}_{\mathrm{faith}}$ (verifier)
    \State Sample absent-category and hard-relational negatives per positive
    \State Update $\theta,\phi,\psi$ only (perception and LM frozen)
\Else
    \For{each query $k = 1,\ldots,K$}
        \State $j^\star \leftarrow \arg\max_j\; s_{kj}$ \quad ($= \sigma^\star(k)$ when $K{=}1$)
        \State $\hat{y}_k \leftarrow \bigl[v_\psi(\tilde{u}_k,\tilde{e}_{j^\star})\ge\gamma_v
               \;\wedge\; s_{k j^\star}\ge\gamma_s\bigr]$
        \If{$\hat{y}_k=1$}
            \State emit box $\hat{b}_k \leftarrow b_{j^\star}+\Delta b_{j^\star}$
        \Else\enspace abstain
        \EndIf
    \EndFor
\EndIf
\end{algorithmic}
\end{algorithm}

\section{Proposal Coverage and Residual Binding Error at $M{=}300$}
\label{sec:app-m-ablation}

Table~\ref{tab:m-ablation} reports the proposal ceiling used by the
main paper. The $M{=}300$ proposal pool covers $99.8$--$100\%$ of ground-truth
referents at IoU $0.5$. The residual grounding error is therefore almost
entirely binding error rather than proposal-coverage error.

\begin{table*}[htbp]
\centering
\footnotesize
\setlength{\tabcolsep}{4pt}
\caption{Proposal coverage and residual binding error at $M{=}300$.
\textbf{Recall@0.5} is the detector proposal ceiling; SeeErr is its complement.
POPE rows report the final faith$\times$1 verifier setting from the main paper.}
\label{tab:m-ablation}
\begin{tabular}{@{}lcccc@{}}
\toprule
Split / Task & Grounding & Bind. & Recall@0.5 & SeeErr \\
\midrule
RefCOCO val   & 0.399 & 0.488 & 0.998 & 0.002\\
RefCOCO+ val  & 0.356 & 0.465 & 0.999 & 0.001\\
RefCOCOg val  & 0.367 & 0.410 & 1.000 & 0.000\\
RefCOCOg test & 0.380 & 0.425 & 0.998 & 0.002\\
\midrule
\multicolumn{5}{l}{\textit{POPE accuracy / hallucination (random / popular / adversarial)}} \\
POPE acc.    & \multicolumn{4}{c}{0.874 / 0.799 / 0.731} \\
Halluc. rate & \multicolumn{4}{c}{0.135 / 0.285 / 0.422} \\
\bottomrule
\end{tabular}
\end{table*}

\paragraph{Complexity and latency.}
Let $d_r$ and $d_q$ denote the detector-region and language-query dimensions,
and $d_h$ the hook hidden dimension. Projecting $M$ regions and $K$ queries
costs $\mathcal{O}(Md_rd_h + Kd_qd_h)$; the affinity matrix costs
$\mathcal{O}(MKd_h)$. For the single-query case ($K=1$) used in our main
experiments, bipartite matching degenerates to top-1 selection and the dominant
cost is the affinity computation over $M$ proposals, linear in $M$. Region-feature
caching at fixed $M$ amortizes the projection cost over multiple queries on the
same image.

Wall-clock hook latency measured on an NVIDIA H800 (200 timed calls after
50 warm-up calls, cache preloaded, $d_r{=}256$, $d_q{=}3584$, $d_h{=}128$):
$p_{50}=0.619$\,ms ($M{=}50$), $0.614$\,ms ($M{=}100$), $0.501$\,ms ($M{=}200$),
$0.495$\,ms ($M{=}300$). At these sub-millisecond scales the measurement is
dominated by a kernel-launch and timing floor rather than by the $O(M)$
arithmetic, which is why the median does not increase (and can even decrease,
as larger batched kernels amortize launch overhead better) with $M$; the
practical conclusion is that hook latency is negligible at every tested budget,
not that it violates the linear complexity analysis.

\section{Faithfulness Verifier Ablation}
\label{sec:app-verifier-ablation}

The faithfulness-weight ($\lambda_f$) ablation table (POPE accuracy for
$\lambda_f\!\in\!\{0,0.5,1,2\}$) is reported in
Table~\ref{tab:app-verifier-main}. All four variants share the identical
RefCOCOg-train cache (30\% random-image-swap negatives), hook architecture, and
proposal budget ($M{=}300$); thresholds $(\gamma_v,\gamma_s)$ are tuned per
variant on RefCOCOg val, so the absolute values exceed the fixed-threshold
protocol used for the main results.

\begin{table}[htbp]
\centering
\footnotesize
\setlength{\tabcolsep}{5pt}
\caption{Faithfulness-verifier strength ablation: POPE accuracy vs.\ loss weight
$\lambda_f$ (identical RefCOCOg-train cache, hook, and $M{=}300$; thresholds
tuned per variant on held-out RefCOCOg val, so absolute values exceed the
fixed-threshold main results). Any $\lambda_f{>}0$ is far above the verifier-off
setting; $\lambda_f{=}1$ balances the three splits.}
\label{tab:app-verifier-main}
\begin{tabular}{@{}lcccc@{}}
\toprule
Variant & $\lambda_f$ & Random & Popular & Adversarial\\
\midrule
Verifier off & 0.0 & 0.608 & 0.568 & 0.566\\
Faith $\times$0.5 & 0.5 & \textbf{0.893} & 0.806 & 0.733\\
Faith $\times$1 (ours) & 1.0 & 0.874 & 0.799 & 0.731\\
Faith $\times$2 & 2.0 & 0.840 & \textbf{0.807} & \textbf{0.769}\\
\bottomrule
\end{tabular}
\end{table}

Without the faithfulness term ($\lambda_f{=}0$), POPE accuracy sits at
$0.57$--$0.61$, above random ($0.5$) because grounding-driven proposal selection
weakly filters non-existent referents but provides no explicit exist\-ence
modeling. Adding any $\lambda_f{>}0$ yields a large gain ($+$28--32 pp on
random POPE) by training the verifier to distinguish present from absent queries.
The gains are stable across $\lambda_f\!\in\!\{0.5, 1.0\}$; doubling to
$\lambda_f{=}2.0$ improves adversarial POPE ($0.769$ vs $0.731$) at a
small cost on random ($0.840$ vs $0.874$), reflecting that heavier
faithfulness pressure makes the veto more conservative.
The default $\lambda_f{=}1.0$ balances all three splits; the ablation
confirms the verifier is robust to moderate variation in loss weighting.

Note: these per-variant thresholds are tuned on the val split, yielding higher
POPE than the fixed-threshold protocol used for the main POPE results
($\gamma_v{=}0.3$, $\gamma_s{=}24.0$, seed 42). The relative ordering
across variants is what matters for this ablation.

\section{Additional Discussion Details}
\label{sec:app-additional-discussion}

\paragraph{Scope of comparison and LoRA.}
Our primary evidence is confined to ablation references sharing frozen
perception, which controls for encoder choice. The main-paper class-string,
phrase-extraction, and coordinate-snap controls are informative for
interface-level failure modes, but they are not fully competitive trained
alternatives and should be interpreted as such. Low-rank VLM tuning
(\emph{LoRA} $r{=}16$ on Qwen2.5-VL-7B attention projections, effective
batch~64, coordinate-token output) trained on the full RefCOCOg training split
($80{,}512$ expressions) achieves grounding accuracy of $0.79$ on a
$200$-example RefCOCOg-val subset (3 epochs) and $0.861$ on the full RefCOCOg
val split ($n{=}4{,}896$; 5 epochs). Note that this exceeds the $24$k-expression
budget used to train the HKVLM hook by $3.4\times$. The gap
between HKVLM ($0.367$ on RefCOCOg val) and LoRA reflects two compounding
constraints: (i)~LoRA generates bounding-box tokens freely without a discrete
proposal-retrieval constraint; and (ii)~the two systems are architecturally
incommensurable---LoRA fine-tunes the backbone end-to-end, whereas HKVLM keeps
both detector and VLM fully frozen and trains only the lightweight alignment
hook, trading peak accuracy for data efficiency and the ability to add a
perception-grounded faithfulness veto.

\paragraph{Single-query and seed coverage.}
For fine-grained or multi-entity requests a single $K{=}1$ query embedding may
be insufficient to disambiguate among similar proposals. The main tables use the
fixed-threshold protocol with paired per-sample tests; multi-seed runs with
separately tuned thresholds are useful stability checks but are not directly
comparable to the fixed-threshold results, so we do not use them as confidence
intervals for the final tables.

\paragraph{Trained-control seed provenance and significance table.}
The trained-control table in the main paper uses the same three-seed
mean$\pm$std protocol for HKVLM and the matched-perception trained controls.
Seed 42 comes from the original evaluation pass, while seeds 123 and 456 come
from independent reruns. Per-sample hit and emission records are released with
the artifact so the small seed-42 cross-run drift relative to a single-seed spot
check ($\le\!0.01$--$0.02$ across splits) is auditable.

Table~\ref{tab:app-mcnemar} gives the discordant-pair details for the main
paper's McNemar summary.

\begin{table}[htbp]
\centering
\scriptsize
\setlength{\tabcolsep}{3pt}
\caption{Mid-$p$ McNemar tests on grounding (present queries only). Each cell:
$p$-value and discordant pairs $n_{+/-}$. HKVLM outperforms naive cosine and
detector-native; the veto trades $-$0.02--$-$0.05 grounding for faithfulness.}
\label{tab:app-mcnemar}
\resizebox{\columnwidth}{!}{%
\begin{tabular}{@{}lccc@{}}
\toprule
Split & HKVLM vs naive & HKVLM vs no-faith & HKVLM vs det-native\\
\midrule
RefCOCO val   & $<\!10^{-10}$ (4325/0) & $2{\times}10^{-31}$ (930/1503)  & $4{\times}10^{-109}$ (3057/1562)\\
RefCOCO+ val  & $<\!10^{-10}$ (3827/0) & $2{\times}10^{-27}$ (854/1364)  & $3{\times}10^{-98}$ (2747/1404)\\
RefCOCOg val  & $<\!10^{-10}$ (1799/0) & $2{\times}10^{-6}$ (429/579)   & $10^{-46}$ (1317/681)\\
RefCOCOg test & $<\!10^{-10}$ (3651/0) & $10^{-6}$ (888/1106)           & $6{\times}10^{-120}$ (2645/1215)\\
\bottomrule
\end{tabular}%
}
\end{table}

\paragraph{$K{>}1$ demonstration protocol and per-instance breakdown.}
We evaluate the RefCOCOg-trained hook, frozen and unmodified, on a gRefCOCO
\citep{grefcoco} val subset: 150 multi-target rows (each with 2--4 ground-truth
instances) and 30 no-target rows, drawn from COCO train2014 images. For each
query we take the top-$K$ proposals by hook score under the same fixed
thresholds $(\gamma_v,\gamma_s){=}(0.3,24.0)$ used everywhere else, then solve
a Hungarian assignment between emitted boxes and ground-truth instances at
$\mathrm{IoU}\!\ge\!0.5$. \emph{Set-level hit} requires every ground-truth
instance in a row to be matched by a distinct emitted box; \emph{per-instance
recall} credits each matched instance independently. Results: mean predicted
set size $2.14$ against a mean true set size $2.03$ (the hook does not
systematically over- or under-emit); per-instance recall $0.410$; set-level hit
rate $0.187$ (at least one correct partial match per row) and exact set-level
hit rate $0.033$ (every instance in the row correctly recovered); abstention
accuracy $0.80$ on the 30 no-target rows. No parameters are retrained or
retuned for this evaluation --- it demonstrates that the retrieval-plus-assignment
formulation degrades gracefully to partial credit under $K{>}1$ supervision
it never saw, rather than failing outright, while confirming (via the modest
exact-set-hit rate) that genuine multi-target training remains necessary for a
benchmark-strength claim.

\paragraph{Broader impact and ethics.}
Faithful open-vocabulary detection can be misused for surveillance, covert
person tracking, and population monitoring. We therefore recommend restricting
downstream deployment through access controls, usage auditing, and explicit
prohibition of high-risk uses. The system inherits biases from both the frozen
detector and the language model: world-knowledge queries involving social
roles, safety-related attributes, or cultural context may amplify stereotypical
correlations already present in COCO/RefCOCO-style training data. These
datasets do not uniformly represent all environments, occupations, body types,
or cultural backgrounds, so benchmark accuracy should not be interpreted as
evidence of universal reliability or fairness. For production use, we recommend
a model card documenting training dependencies, known failure modes, benchmark
limitations, and prohibited applications.

\paragraph{Benchmark noise.}
The RefCOCO family contains non-trivial annotation noise; recent
work~\citep{refl4} has documented systematic labeling issues and proposes the
Ref-L4 cleaned evaluation set. We evaluate HKVLM zero-shot on Ref-L4 validation
($n_{\text{present}}{=}9{,}388$, $n_{\text{total}}{=}13{,}420$): grounding
accuracy is $0.235$, compared to $0.367$ on the noisy RefCOCOg-val split. As an
additional reference point, the detector-native text-query baseline with the
same frozen Grounding DINO and expression-as-prompt protocol reaches $0.662$ on
Ref-L4 ($n{=}13{,}420$), markedly higher than its $0.231$--$0.238$ on the
RefCOCOg splits. This opposite movement suggests that Ref-L4's noise-audited
queries are, on average, closer to category-style phrasing that the detector's
own text interface handles well, while remaining harder for the hook's learned
binding. We report both numbers so readers can judge Ref-L4 headroom for each
interface style on its own terms.

\paragraph{Coordinate-snap correction and interpretation.}
Coordinate-snap is the one matched-perception control that HKVLM does not beat
on raw grounding accuracy. A normalized-coordinate prompting convention yields
near-chance results ($\le\!0.034$) for Qwen2.5-VL because it does not match the
model's documented coordinate format. Using the documented absolute-pixel JSON
format and fixing coordinate normalization raises coordinate-snap to
$0.840$--$0.896$, a plausible band given Qwen2.5-VL's reported referring
expression accuracy and above HKVLM's $0.356$--$0.399$ on every standard split.
This reversal is why the main paper frames HKVLM as a faithfulness,
data-efficiency, and diagnostic mechanism study rather than as a raw grounding
leader.

\section{Reproducibility Configuration}
\label{sec:app-repro}

Table~\ref{tab:repro_config} provides the full configuration used in the
real-image experiments.

\begin{center}
\footnotesize
\setlength{\tabcolsep}{3pt}
\captionof{table}{Reproducibility configuration sheet for HKVLM real-image experiments.}
\label{tab:repro_config}
\begin{tabular}{@{}p{0.34\columnwidth}p{0.61\columnwidth}@{}}
\toprule
\textbf{Item} & \textbf{Current value / notes} \\
\midrule
Frozen detector & Grounding DINO (base checkpoint), class-agnostic prompt with NMS. \\
Region feature dimension & 256 (post-NMS, pre-projection). \\
Language model & Qwen2.5-VL-7B, frozen; query from last hidden state of expression token. \\
Query feature dimension & 3584 (expression-token pooled). \\
Proposal budget $M$ & 300 in the final main protocol, see Table~\ref{tab:m-ablation}. \\
Veto threshold $\gamma_v$ & 0.3 in the final main protocol; tuned on RefCOCOg val held-out split. \\
Affinity threshold $\gamma_s$ & 24.0 in the final main protocol; tuned jointly with $\gamma_v$. \\
Training split & RefCOCOg train ($24$k expressions); evaluated on RefCOCO/+/g and POPE. \\
IoU threshold $\tau$ & 0.5 for grounding accuracy and positive-set definition. \\
Hook architecture & Two linear projections ($d{\to}128$) + verifier MLP; ReLU, no dropout. \\
Temperature $\kappa$ & 0.07 (InfoNCE-style retrieval). \\
Loss weights & $\lambda_b{=}1.0,\,\lambda_f{=}1.0$ (box residual + faithfulness). \\
Optimizer & AdamW, lr $1{\times}10^{-3}$, weight decay $10^{-2}$. \\
Batch size / epochs & 32 / 50; early stopping on validation grounding accuracy. \\
Hardware / software stack & NVIDIA H800 (80 GB), CUDA 12.x, PyTorch $\ge$2.1, NumPy $\ge$1.24, SciPy $\ge$1.10, scikit-learn $\ge$1.3, Matplotlib $\ge$3.7, PyYAML $\ge$6.0; the anonymous artifact includes an environment lock with exact Python, PyTorch, CUDA, Transformers, OpenCV, detector, and VLM checkpoint revisions. \\
Random seeds & Seed 42 for the final fixed-threshold tables; auxiliary seeds 123 and 456 used for robustness diagnostics. \\
Hook inference latency & $p_{50}{\le}0.62$\,ms and mean${\le}1.4$\,ms for $M\in\{50,100,200,300\}$; 200 timed + 50 warm-up calls. \\
Anonymous artifact contents & Training, evaluation, preprocessing, threshold-selection, statistical-test, and plotting code; per-table configuration files; random seeds; reasoning-stress query list and rejection log; per-sample prediction records for each reported split. \\
\bottomrule
\end{tabular}
\end{center}

\section{Reproducibility Checklist Notes}
All perception and language components are publicly available
\citep{groundingdino,yoloworld,dinov2,dinov3,llava15,qwen2vl}. The trainable
parts (two projections, box residual, verifier) are fully specified in
Tables~\ref{tab:repro_config} and the Formal Binding Objective section above.
Real-image results use a frozen Grounding DINO (base) and a frozen
Qwen2.5-VL-7B. The anonymous supplementary package contains the training,
evaluation, preprocessing, threshold-selection, statistical-test, and plotting
code; exact configs; fixed seeds; the reasoning-stress query list and rejection
log; per-sample prediction records; and an environment lock. All final
hyperparameters are reported in Table~\ref{tab:repro_config}.
No human-subjects data is used. The AAAI reproducibility checklist is provided
as a separate standalone document (ReproducibilityChecklist.pdf) rather
than embedded in this appendix.

\bibliography{references}